\def\isarxiv{1} 

\ifdefined\isarxiv
\documentclass[11pt]{article}
\usepackage[numbers]{natbib}
\else
\documentclass[twoside]{article}
\usepackage[accepted]{aistats2025}
\usepackage[round]{natbib}

\fi

\usepackage{amsmath}
\usepackage{amsthm}
\usepackage{amssymb}
\usepackage{algorithm}
\usepackage{subfig}
\usepackage{algpseudocode}
\usepackage{graphicx}
\usepackage{grffile}
\usepackage{wrapfig,epsfig}
\usepackage{url}
\usepackage{xcolor}
\usepackage{epstopdf}

\usepackage{bbm}
\usepackage{dsfont}

\allowdisplaybreaks

\ifdefined\isarxiv
\renewcommand*{\citet}{\cite}
\renewcommand*{\citep}{\cite}

\usepackage{tikz}
\usepackage{hyperref}  
\hypersetup{colorlinks=true,citecolor=blue,linkcolor=blue} 
\usetikzlibrary{arrows}
\usepackage[margin=1in]{geometry}

\else

\usepackage{microtype}
\usepackage{hyperref}
\definecolor{mydarkblue}{rgb}{0,0.08,0.45}
\hypersetup{colorlinks=true, citecolor=mydarkblue,linkcolor=mydarkblue}

\fi
 
\graphicspath{{./figs/}}

\theoremstyle{plain}
\newtheorem{theorem}{Theorem}[section]
\newtheorem{lemma}[theorem]{Lemma}
\newtheorem{definition}[theorem]{Definition}

\newtheorem{fact}[theorem]{Fact}
\newtheorem{remark}[theorem]{Remark}

\newcommand{\wt}{\widetilde}

\newcommand{\N}{\mathcal{N}}
\newcommand{\R}{\mathbb{R}}

\makeatletter
\newcommand*{\RN}[1]{\expandafter\@slowromancap\romannumeral #1@}
\makeatother

\usepackage{lineno}

\begin{document}

\ifdefined\isarxiv

\date{}

\title{Bypassing the Exponential Dependency: Looped Transformers Efficiently Learn In-context by Multi-step Gradient Descent}
\author{
Bo Chen\thanks{\texttt{ bc7b@mtmail.mtsu.edu}. Middle Tennessee State University.}
\and
Xiaoyu Li\thanks{\texttt{
xiaoyu.li2@student.unsw.edu.au}. University of New South Wales.}
\and 
Yingyu Liang\thanks{\texttt{
yingyul@hku.hk}. The University of Hong Kong. \texttt{
yliang@cs.wisc.edu}. University of Wisconsin-Madison.} 
\and
Zhenmei Shi\thanks{\texttt{
zhmeishi@cs.wisc.edu}. University of Wisconsin-Madison.}
\and 
Zhao Song\thanks{\texttt{ magic.linuxkde@gmail.com}. The Simons Institute for the Theory of Computing at the University of California, Berkeley.}
}

\else

\runningtitle{Looped Transformers Efficiently Learn In-context by Multi-step Gradient Descent}

\runningauthor{Bo Chen, Xiaoyu Li, Yingyu Liang, Zhenmei Shi, Zhao Song}

\twocolumn[
\aistatstitle{Bypassing the Exponential Dependency: Looped Transformers Efficiently Learn In-context by Multi-step Gradient Descent}

\aistatsauthor{
    \textbf{Bo Chen}$^1$ \And
    \textbf{Xiaoyu Li}$^2$ \And
    \textbf{Yingyu Liang}$^{3,4}$ \And
    \textbf{Zhenmei Shi}$^4$ \And
    \textbf{Zhao Song}$^5$
}

\aistatsaddress{
    $^1$Middle Tennessee State University \\
    $^2$University of New South Wales \\
    $^3$The University of Hong Kong \\
    $^4$University of Wisconsin-Madison \\
    $^5$The Simons Institute for the Theory of Computing at the University of California, Berkeley
}

]

\fi

\ifdefined\isarxiv
\begin{titlepage}
  \maketitle
  \begin{abstract}
In-context learning has been recognized as a key factor in the success of Large Language Models (LLMs). It refers to the model's ability to learn patterns on the fly from provided in-context examples in the prompt during inference. Previous studies have demonstrated that the Transformer architecture used in LLMs can implement a single-step gradient descent update by processing in-context examples in a single forward pass. Recent work has further shown that, during in-context learning, a looped Transformer can implement multi-step gradient descent updates in forward passes. However, their theoretical results require an exponential number of in-context examples, $n = \exp(\Omega(T))$, where $T$ is the number of loops or passes, to achieve a reasonably low error. In this paper, we study linear looped  Transformers in-context learning on linear vector generation tasks. We show that linear looped Transformers can implement multi-step gradient descent efficiently for in-context learning. Our results demonstrate that as long as the input data has a constant condition number, e.g., $n = O(d)$, the linear looped Transformers can achieve a small error by multi-step gradient descent during in-context learning. Furthermore, our preliminary experiments validate our theoretical analysis. Our findings reveal that the Transformer architecture possesses a stronger in-context learning capability than previously understood, offering new insights into the mechanisms behind LLMs and potentially guiding the better design of efficient inference algorithms for LLMs.

  \end{abstract}
  \thispagestyle{empty}
\end{titlepage}

{\hypersetup{linkcolor=black}
\tableofcontents
}
\newpage

\else

\begin{abstract}
In-context learning has been recognized as a key factor in the success of Large Language Models (LLMs). It refers to the model's ability to learn patterns on the fly from provided in-context examples in the prompt during inference. Previous studies have demonstrated that the Transformer architecture used in LLMs can implement a single-step gradient descent update by processing in-context examples in a single forward pass. Recent work has further shown that, during in-context learning, a looped Transformer can implement multi-step gradient descent updates in forward passes. However, their theoretical results require an exponential number of in-context examples, $n = \exp(\Omega(T))$, where $T$ is the number of loops or passes, to achieve a reasonably low error. In this paper, we study linear looped  Transformers in-context learning on linear vector generation tasks. We show that linear looped Transformers can implement multi-step gradient descent efficiently for in-context learning. Our results demonstrate that as long as the input data has a constant condition number, e.g., $n = O(d)$, the linear looped Transformers can achieve a small error by multi-step gradient descent during in-context learning. Furthermore, our preliminary experiments validate our theoretical analysis. Our findings reveal that the Transformer architecture possesses a stronger in-context learning capability than previously understood, offering new insights into the mechanisms behind LLMs and potentially guiding the better design of efficient inference algorithms for LLMs.
\end{abstract}

\fi

\section{INTRODUCTION}
Large Language Models (LLMs) have gained immense success and have been widely used in our daily lives, e.g., GPT4 \citep{o23}, Claude 3.5 \citep{c24}, and so on, based on its Transformer architecture \citep{vsp+17}. 
One core emergent ability of LLMs is In-Context Learning (ICL) \citep{bmr+20}. 
During ICL, the user provides an input sequence (prompts) containing some question-answer pairs as in-context examples and the goal query that the user cares about, where the examples and query are drawn from an unknown task.  
The LLMs can in-context learn from these examples and generate the correct answer for the goal query without any parameter update. 
Notably, these unknown tasks may never be seen by LLMs during their pre-training and post-training, e.g., synthetic task \citep{whl+23}. 
Thus, many believe that the in-context learning mechanism is different from supervised learning or unsupervised learning, where the latter may focus on feature learning, while the ICL may perform some algorithm to learn.  
For instance, the Transformer can implement algorithm selection \citep{bcw+24} or gradient descent \citep{dld+22,vnr+23} by in-context learning forward pass. 

Many works have tried to understand how Transformers perform single-step gradient descent \citep{zfb23,acds24,mhm23,ccs23,bcw+24,hcl23,lwl+23,ghm+24,wzc+24}. They study one-layer linear Transformer in-context learning on linear regression tasks and show that the one-layer linear Transformer can perform single gradient updates based on the in-context examples during a forward pass. 
Recent work \citep{gsr+24} has further shown that a linear looped Transformer can implement multi-step gradient descent updates with multiple forward passes, meaning that Transformers can express multi-step algorithms during in-context learning. 
However, their theoretical results require an exponential number of in-context examples, $\exp(\Omega(T ))$, where $T$ is the number of loops or passes, to achieve a reasonably low error for linear regression tasks. 
This violates the intuition that more gradient descent updates lead to better performance. 

Thus, it is natural to ask the following question:
\begin{center}
    {\it Is it necessary to use an exponential number of examples for Transformers to implement multi-step gradient descent during in-context learning? }
\end{center}

In this work, we study
linear looped Transformers (Definition~\ref{def:loop}) in-context learning on linear vector generation tasks (Definition~\ref{def:input}), which is as hard as linear regression.
We show that the linear looped Transformer can efficiently perform multi-step gradient descent as long as in-context examples are well-conditioned. We present our main result in the following theorem. 
\begin{theorem}[Main result. Informal version of Theorem~\ref{thm:main}]\label{thm:main:informal}
Let $T$ be the number of loops, $n$ be the number of in-context examples, and $d$ be the feature dimension. 
Let $(X, y) \in \R^{n \times d} \times \R^{n}$ be the in-context examples. Let $\kappa$ be the condition number of $X^\top X$ (Definition~\ref{def:condition}). Then, given a query $\alpha \in \R$, the linear looped Transformer can predict a vector output with error $\le |\alpha| \cdot \exp(-\frac{T}{2\kappa})$ by explicitly performing multi-step gradient descent in its hidden states. 
\end{theorem}

In Theorem~\ref{thm:main:informal}, as long as the condition number is constant, we can see that the linear looped Transformer will perform better when the loop number is increasing, i.e., the error will exponentially decay to $0$. 
Informally, for a small constant $\epsilon$, if we draw $X$ from Gaussian distribution, the condition number of $X^\top X$ will be $1 \le \kappa \le 1 + O(\epsilon)$ when $n \ge \Omega(d/\epsilon^2)$.
Thus, informally, we only need $O(d)$ numbers of in-context examples to guarantee a good performance. 

Furthermore, our preliminary experiments (Section~\ref{sec:expr}) validate the above arguments and our theoretical analysis. 
The main intuition of our analysis is that we find that a linear looped Transformer can explicitly perform gradient descent in its hidden states (Lemma~\ref{lem:single_linear_output_general} and Theorem~\ref{thm:transfomer_output}). Thus, the error analysis can be directly solved by the standard convex optimization technique (Theorem~\ref{thm:optimz_polyn}). 

\subsection{Our Contributions}
\begin{itemize}
    \item We study linear looped Transformers in-context learning on linear vector generation tasks (Definition~\ref{def:icl_task}), which is as hard as linear regression. 
    \item We find that linear looped Transformers can explicitly perform gradient descent in their hidden states (Lemma~\ref{lem:single_linear_output_general} and Theorem~\ref{thm:transfomer_output}). 
    \item We demonstrate that linear looped Transformers can efficiently perform multi-step gradient descent as long as in-context examples are well-conditioned, e.g., $n=O(d)$, where the error will exponentially decay to $0$ after $T$ loops (Theorem~\ref{thm:main}).  
    \item Our preliminary experiments on synthetic data validate our main theoretical results (Section~\ref{sec:expr}). 
\end{itemize}

\subsection{Roadmap} This paper is structured as follows: we begin with a review of related work in Section~\ref{sec:related_work}, followed by essential definitions and foundational concepts in Section~\ref{sec:pre}. Section~\ref{sec:gradient} delves into the gradient computation analysis within the Looped Transformer architecture, examining both individual layer computations and the full looped structure. In Section~\ref{sec:error}, we analyze the error convergence of looped transformers under strong convexity and smoothness conditions, providing an upper bound on error after $T$ gradient descent iterations for linear vector generation. Section~\ref{sec:expr} presents our experimental results and findings.
Finally, we conclude in Section~\ref{sec:conclusion}.
\section{RELATED WORK}\label{sec:related_work}
This section briefly reviews the related research work on Large Language Models (LLM), In-Context Learning (ICL), and looped transformers. These topics have a close connection to our work.

\subsection{Large Language Models} Neural networks based on the Transformer architecture~\citep{vsp+17} have swiftly become the dominant paradigm in machine learning for natural language processing applications. Expansive Transformer models, trained on diverse and extensive datasets and comprising billions of parameters, are called large language models (LLM) or foundation models~\citep{bha+21}. Examples include BERT~\citep{dclt19}, PaLM~\citep{cnd+22}, Llama~\citep{tli+23}, ChatGPT~\citep{chatgpt}, GPT4~\citep{o23}, among others. These LLMs have demonstrated remarkable general intelligence capabilities~\citep{bce+23} across various downstream tasks. 

Researchers have developed numerous adaptation techniques to optimize LLM performance for specific applications. These include methods such as adapters~\citep{eyp+22,zhz+23,ghz+23,zjk+23,cls24,c24_sorsa}, calibration approaches~\citep{zwf+21,cpp+23}, multitask fine-tuning strategies~\citep{gfc+21a,zzj+23a,vnr+23,zzj+23b}, prompt tuning techniques~\citep{gfc+21b,lac+21}, scratchpad approaches~\citep{naa+21}, instruction tuning methodologies~\citep{ll21,chl+22,mkd+22}, symbol tuning~\citep{jla+23}, black-box tuning~\citep{ssy+22}, reinforcement learning from the human feedback (RLHF)~\citep{owj+22}, chain-of-thought reasoning~\citep{wws+22,ksk+22,sdj+23,zmc+24} and various other strategies. 

And here are more related works aiming at enhancing model efficiency without compromising performance, such as \citep{dswy22_coreset,qss23_gnn,ssx23_nns,gms23,dms23_spar,lls+24_prune,lss+24_multi_layer,lssy24,whhl24,whl+24,xhh+24,hyw+23,hcl+24,hcw+24,hlsl24,hwl24,hwl+24,hwsl24,llss24_sparsegpt,lls+24_io,lll+25_loop,syz23,smn+24,ssz+24_pruning,ccl+25, cll+24_rope,chl+24_rope_grad,cll+24_ssm,cgl+25_homo,cll+25,cll+25_var,cls+25,zly+25,kll+25_var}.

\subsection{In-context Learning} 
A significant capability that has emerged from Large Language Models (LLMs) is in-context learning (ICL)~\citep{bmr+20,jyr+22,swxl24}. This feature allows LLMs to generate predictions for new scenarios when provided with a concise set of input-output examples (referred to as a prompt) for a specific task, without requiring any modification to the model's parameters. ICL has found widespread application across various domains, including reasoning~\citep{znl+22}, self-correction~\citep{pr23}, machine translation~\citep{azl+22} and many others. 

Numerous studies have focused on enhancing the ICL and zero-shot capabilities of LLMs~\citep{mlzh21,wma+22,jmv+22,ilp+22}. 
A substantial body of research has been dedicated to investigating the underlying mechanisms of transformer learning~\citep{gsb+21,sapt22,gtl+22,jsl+22,ag23,hmsp23,llr23,al23,zsc+23,twcd23,twz+23,zbl+23,avd+23,xsl24,gll+24b,lls+24_grok,lls+24_conv,lssz24_tat,gls+24b} and in-context learning~\citep{dld+22,mhm23,amfs23,yfh+23,kxhs23,pgcc23,lwl+23,lsg+23,edj+23,zzy+23,zfb23,hcl23,ccs23,ww23,wzc+24,ghm+24,r24,swxl24,wsh+24} through both empirical and theoretical approaches. Building upon these insights, our analysis extends further to elucidate ICL implementing multi-step gradient descent.

\subsection{Looped Transformer}
The concept of recursive inductive bias was first introduced by \cite{dgv+18} into Transformers. Looping Transformers are also related to parameter-efficient weight-tying Transformers \citep{ccw+21,gsr+24} theoretically showed that the recursive structure of Looped Transformers enables
them to function as Turing machines. \cite{ylnp23} demonstrates that increasing the number of loop iterations can improve performance on some tasks. Recent studies \citep{af23,gyw+24} have provided theoretical insights into the emulation capabilities of specific algorithms and their convergence during training, with a particular emphasis on in-context learning. 

However, the expressive capacity of Looped Transformers or Looped Neural Newrok~\cite{lss+24_relu,kll+24} in function approximation and their associated approximation rates remain largely unexplored territories. \cite{yhl+25} studies enhancing autoregressive chain-of-thought through loop-aligned reasoning.
In this work, we consider Looped Linear Transformers.

\begin{table*}[!ht]
    \caption{There is a strict correspondence between the symbols we use and the common symbols of the general standard transformer attention. The specific choice of $Q$ and $P$ follows the setting in previous work \cite{gsr+24,hcl23} for simplicity. }
    \label{tab:notation}
    \centering
    \begin{tabular}{c|c|c|c}
\hline
\textbf{Meaning} & \textbf{Dimension} & \textbf{Standard Notation} & \textbf{Our Notation} \\
\hline
{Input} & $n \times d$ & $X$ & $Z$ \\
{Query-key weight matrix} & $d \times d$ & $W_Q W_K^\top$ & $Q$ \\
{Value-output weight matrix} & $d \times d$ & $W_V W_O^\top$ & $P$ \\
{Query-key matrix} & $n \times n$ & $QK^\top$ = $X W_Q W_K^\top X^\top$ & $Z Q X^\top Z$ \\
\hline
    \end{tabular}
\end{table*}

\section{PRELIMINARY}\label{sec:pre}

In this section, we present some preliminary concepts and definitions of our paper. In Section~\ref{sec:pre:notations}, we introduce some basic notations used in our paper. In Section~\ref{sec:pre:problem_set}, we defined some important variables to set up our problem.
\subsection{Notations}\label{sec:pre:notations}
For any integer $n$, we define $[n] = \{1,2, \dots, n \}$.
For two vectors $x \in \R^n$ and $y \in \R^n$, we use $\langle x, y \rangle$ to denote the inner product between $x,y$, i.e., $\langle x, y \rangle = \sum_{i=1}^n x_i y_i$.
We use $e_i$ to denote a vector where only $i$-th coordinate is $1$, and other entries are $0$.
For each $a, b \in \R^n$, we use $a \circ b \in \R^n$ to denote the Hardamard product, i.e. the $i$-th entry of $(a\circ b)$ is $a_i b_i$ for all $i \in [n]$.
We use ${\bf 1}_n$ to denote a length-$n$ vector where all the entries are ones.
We use $x_{i,j}$ to denote the $j$-th coordinate of $x_i \in \R^n$.
We use $\|x\|_p$ to denote the $\ell_p$ norm of a vector $x \in \R^n$, i.e. $\|x\|_1 := \sum_{i=1}^n |x_i|$, $\|x\|_2 := (\sum_{i=1}^n x_i^2)^{1/2}$, and $\|x\|_{\infty} := \max_{i \in [n]} |x_i|$.
For $n > d$, for any matrix $A \in \R^{n \times d}$, we use $\|A\|$ to denote the spectral norm of $A$, i.e. $\|A\|:=\sup_{x\in \R^d} \|Ax\|_2 / \|x\|_2$.
For a matrix $A$, we use $\lambda_{\min}(A)$ and $\lambda_{\max}(A)$ to denote the minimum and maximum eigenvalue of $A$, respectively.

\subsection{In-context Learning}\label{sec:pre:problem_set}
First, we introduce some definitions of in-context examples and their labels.

\begin{definition}[In-context prompt data]
Let $X \in \R^{n \times d}$ be the input data with $n$ tokens and each token with feature dimension $d$, i.e, $X = [x_1, x_2, \dots, x_n]^\top$, where $x_i \in \R^{d}$ for any $i \in [n]$.     
\end{definition}

\begin{definition}[In-context prompt label]
    Assume $\theta^*$ uniformly draw form $d$ dimension unit sphere $\mathbb{S}^{d}$, where $\|\theta^*\|_2 = 1$. 
    Let labels be 
    \begin{align*}
        y := X \theta^* \in \R^{n}.
    \end{align*}
\end{definition}

Note that the $\theta^*$ is unseen to the model. Then, our vector generation tasks are defined as follows. 

\begin{definition}[In-context task]\label{def:icl_task}
Given in-context prompt examples/data $X \in \R^{n \times d}$ with their labels $y = X \theta^*$, where $\theta^*$ is unseen to model, and given a query $\alpha \neq 0 \in \R$, 
the task is to output/generate a vector $v$ such that $\langle v, \theta^*\rangle$ is as close to $\alpha$ as possible. 
Thus, the prediction error is 
\begin{align*}
    |\langle v, \theta^* \rangle - \alpha|.
\end{align*}
\end{definition}

We remark that our vector generation task is as hard as the linear regression task, as they are dual problems. Solving any one of them requires estimating the $\theta^*$. 

Combining all of the above, we have the ICL prompt/input data for the model. 
\begin{definition}[Input data]\label{def:input}
Let input be 
\begin{align*}
Z^{(0)} := \begin{bmatrix}
    X & y \\
    q^{(0)\top} & \alpha 
\end{bmatrix} \in \R^{(n+1) \times  (d+1)  }.
\end{align*}
\end{definition}
In Definition~\ref{def:input}, the model will iteratively update $q^{(0)} \in \R^{d}$ and make it as the generation output. In this work, we initialize $q^{(0)}$ as ${\bf 0}_d$. 

\subsection{Linear Looped Transformer}

In line with recent work by \citet{gsr+24} and \citet{acds24}, we consider a linear self-attention model, formally defined as follows:

\begin{definition}[Linear attention]\label{def:lin_attn}
Let $Q, P \in \R^{d \times d}$ be the query-key matrix and the value-output matrix. Let $Z \in \R^{n \times d}$ be input data. 
The linear attention is defined as 
\begin{align*}
    \mathsf{Attn}(Z; Q, P) :=
    (M \circ (Z Q Z^\top )) Z P,
\end{align*}
where $M =  \begin{bmatrix}
            \mathbf{0}_{n\times n} & \mathbf{0}_{n \times 1} \\
            \mathbf{1}_{1\times n} & 0 
    \end{bmatrix}$ is a causal attention mask for text generation.
\end{definition}
In Definition~\ref{def:lin_attn}, we combine the query matrix and key matrix as $Q$ and combine the value matrix and output matrix as $P$ for simplicity following previous works \citep{zfb23,gsr+24}. We refer readers to Table~\ref{tab:notation} for more details about connection between our notation and standard attention. 

Building upon this, we introduce the concept of a Linear Looped Transformer \citep{gsr+24}:
\begin{definition}[Linear looped transformer]\label{def:loop}
Let $T$ be the loop number. Let $\eta^{(t)} >0$ for any $t\in \{0, 1, \dots, T-1\}$. The linear looped transformer $\mathsf{TF}(Z^{(0)}; Q, P)$ is defined as
\begin{align*}
    & ~ Z^{(t)} := Z^{(t-1)} - \eta^{(t-1)} \mathsf{Attn}(Z^{(t-1)};Q,P), ~\forall t \in [T] \\
    & ~ \mathsf{TF}(Z^{(0)}; Q, P) := - {Z^{(T)}_{n+1, 1:d}}^\top.
\end{align*}
\end{definition}

The Looped Transformer is to simulate a real multiple-layer Transformer with residual connections \citep{hzrs16}, where $\eta$ represents the weights of the residual components. 

\begin{remark}
Our settings are more practical than \cite{gsr+24} in the following sense.
\begin{itemize}
    \item We have a more practical casual attention mask used in generation. Our mask requires Hardamard product, the same as the standard attention, while \cite{gsr+24} uses matrix product mask. 
    \item Our model does not have prior knowledge of $X$ distribution, while the model in \cite{gsr+24} knows the distribution of $X$, i.e., distribution free. In practice, the LLMs do not know any information about in-context examples. 
\end{itemize}
\end{remark}

\subsection{Linear Regression with Gradient Descent}
In this section, we introduce some key concepts of linear regression with gradient descent.

\begin{definition}[Linear regression]\label{def:grad}
Given $X \in \R^{n \times d}$, $ y \in \R^{n}$, and $\theta \in \R^d$, the linear regression loss function is defined as
\begin{align*}
    \ell(\theta) := 0.5 \|y - X \theta \|_2^2.
\end{align*}
Then, the gradient formulation is
\begin{align}
    \nabla_\theta \ell(\theta) = X^\top X \theta - X^\top y \in \R^d. \label{eq:grad}
\end{align}
For any $t \in \mathbb{N}_+$, let $\eta^{(t)} > 0$ and by Gradient Descent, we have
\begin{align}
    \theta^{(t)} := & ~ \theta^{(t-1)} - \eta^{(t-1)} (X^\top X \theta^{(t-1)} - X^\top y). \label{eq:gd_update}
\end{align}
\end{definition}

We define the optimizer below. 
\begin{definition}
    Let the optimizer of $\ell(\theta)$ be
    \begin{align*}
        \wt{\theta} = (X^\top X)^{-1} X^\top y .
    \end{align*}
\end{definition}

In this work, we consider $n > d$, assuming $X^\top X$ is inevitable. Then, we have $\theta^* = \wt{\theta}$.
\begin{lemma}[Forklore]
    In the realizable setting (no noise term), if $n > d$, then $\theta^* = \wt{\theta}$.
\end{lemma}

Finally, we define the condition number, which will be used in our final convergence bound. 
\begin{definition}\label{def:condition}
    We define the condition number of input data as 
    \begin{align*}
        \kappa := \frac{\lambda_{\max}(X^\top X)}{\lambda_{\min}(X^\top X)}.
    \end{align*}
\end{definition}

\section{GRADIENT COMPUTATION IN LOOPED TRANSFORMER}\label{sec:gradient}
In this section, we present a comprehensive analysis of the gradient computation process within the Looped Transformer architecture. Our investigation begins with an examination of computations in individual layers and subsequently extends to the full looped structure. This approach allows us to build a nuanced understanding of the Looped Transformer's behavior, starting from its fundamental components.

We commence our analysis by establishing a crucial result regarding the output of a single layer in our Looped Transformer model. This foundational lemma serves as a cornerstone for our subsequent derivations and provides valuable insights into the model's inner workings.
\begin{lemma}[Single layer output]\label{lem:single_linear_output_general}
Let $Z^{(0)}$ be defined in Definition~\ref{def:input}. Let $Q = I_{d+1, d+1}$. Let $P=\begin{bmatrix}
            I_{d\times d} & \mathbf{0}_{d \times 1} \\
            \mathbf{0}_{1\times d} & 0 
    \end{bmatrix}$. 
    Let causal attention mask be
$M =  \begin{bmatrix}
            \mathbf{0}_{n\times n} & \mathbf{0}_{n \times 1} \\
            \mathbf{1}_{1\times n} & 0 
    \end{bmatrix}$.
Then, we have 
\begin{align*}
\mathsf{Attn}(Z^{(0)}; Q, P) = \begin{bmatrix}
    \mathbf{0}_{n\times d}  & \mathbf{0}_{n\times 1} \\
    q^{(0)\top} X^\top X  + \alpha y^\top X & 0
    \end{bmatrix}.
\end{align*}
\end{lemma}

\begin{proof}
We can show that
\begin{align*}
    & ~\mathsf{Attn}(Z^{(0)}; Q, P)\\ 
    = & ~ (M \circ (Z^{(0)} Q (Z^{(0)\top}) )) Z^{(0)} P  \\
    = & ~ (M \circ \begin{bmatrix}
    X & y \\
    q^{(0)\top} & \alpha 
\end{bmatrix} 
    \begin{bmatrix}
    X^\top &  q^{(0)} \\
    y^\top & \alpha 
\end{bmatrix}) Z^{(0)} P \\
    = & ~ (M \circ \begin{bmatrix}
    X X^\top + yy^\top  & X q^{(0)} + \alpha y\\
    q^{(0)\top} X^\top  + \alpha y^\top & q^{(0)\top } q^{(0)} + \alpha^2
    \end{bmatrix}) Z^{(0)}P
    \\
    = & ~ \begin{bmatrix}
    \mathbf{0}_{n\times n}  & \mathbf{0}_{n\times 1} \\
    q^{(0)\top} X^\top  + \alpha y^\top & 0
    \end{bmatrix} Z^{(0)}P \\
    = & ~ \begin{bmatrix}
    \mathbf{0}_{n\times n}  & \mathbf{0}_{n\times 1} \\
    q^{(0)\top} X^\top  + \alpha y^\top & 0
    \end{bmatrix} \begin{bmatrix}
    X & y \\
    q^{(0)\top} & \alpha 
\end{bmatrix} P \\
    = & ~ \begin{bmatrix}
    \mathbf{0}_{n\times d}  & \mathbf{0}_{n\times 1} \\
    q^{(0)\top} X^\top X  + \alpha y^\top X & q^{(0)\top} X^\top y  + \alpha y^\top y
    \end{bmatrix} P \\
    = & ~ \begin{bmatrix}
    \mathbf{0}_{n\times d}  & \mathbf{0}_{n\times 1} \\
    q^{(0)\top} X^\top X  + \alpha y^\top X & q^{(0)\top} X^\top y  + \alpha y^\top y
    \end{bmatrix} \\
    & ~ \cdot 
    \begin{bmatrix}
            I_{d\times d} & \mathbf{0}_{d \times 1} \\
            \mathbf{0}_{1\times d} & 0 
    \end{bmatrix} \\
    = & ~ \begin{bmatrix}
    \mathbf{0}_{n\times d}  & \mathbf{0}_{n\times 1} \\
    q^{(0)\top} X^\top X  + \alpha y^\top X & 0
    \end{bmatrix}
\end{align*}
where the first step follows from Definition~\ref{def:lin_attn}, the second step follows from Definition~\ref{def:input}, and the rest steps directly follow from the matrix multiplication.
\end{proof}

This result illuminates the specific form of the attention mechanism's output in a single layer, which is essential for understanding the model's overall behavior, where the output only has non-zero terms in the position of $q^{(0)}$ and the format is close to Eq.~\eqref{eq:grad}. 

Having characterized the behavior of a single layer, we now extend our analysis to encompass the full-looped transformer structure. Turning our attention to the core of our analysis with the groundwork laid for single-layer computations. The following theorem establishes a crucial relationship between the transformer's output and the iteratively updated parameters:

\begin{theorem}\label{thm:transfomer_output}
Let $\alpha \neq 0 \in \R$. Let $\theta^{(0)} = -\frac{1}{\alpha} q^{(0)}$ . Let $\theta^{(i)}$ correspondingly be defined in Definition~\ref{def:grad} for any $i \in \{2,3, \dots, T\}$. 
Let $\mathsf{TF}(Z^{(0)}; Q, P)$ be defined in Definition~\ref{def:loop}.  
We have 
\begin{align*}
    \mathsf{TF}(Z^{(0)}; Q, P) = \alpha \theta^{(T)}.
\end{align*}
\end{theorem}
\begin{proof}
    We will prove this theorem by induction on $t$, where $t \in [T]$.
    From Lemma~\ref{lem:single_linear_output_general}, we have:
    \begin{align*}
    \mathsf{Attn}(Z^{(0)}; Q, P) = \begin{bmatrix}
    \mathbf{0}_{n\times d}  & \mathbf{0}_{n\times 1} \\
    q^{(0)\top} X^\top X  + \alpha y^\top X & 0
    \end{bmatrix}.
    \end{align*}
    For $t = 1$, we can calculate:
    \begin{align*}
        Z^{(1)} = & ~ Z^{(0)} - \eta^{(0)} \mathsf{Attn}(Z^{(0)}; Q, P) \\
        = & ~ \begin{bmatrix} X & y \\ q^{(0)\top} & \alpha \end{bmatrix} - \eta^{(0)}\begin{bmatrix} {\bf 0}_{n\times d} & {\bf 0}_{n\times 1} \\ q^{(0)\top} X^\top X + \alpha y^\top X & 0 \end{bmatrix} \\
        = & ~ \begin{bmatrix} X & y \\ q^{(0)\top} - \eta^{(0)}(q^{(0)\top} X^\top X + \alpha y^\top X) & \alpha \end{bmatrix}
    \end{align*}
    Where the first step follows from Definition~\ref{def:loop}, the second step follows from Definition~\ref{def:input} and Definition~\ref{def:lin_attn}, the third step follows from basic algebra. Then we extract $q^{(0)\top} - \eta^{(0)}(q^{(0)\top} X^\top X + \alpha y^\top X)$, we have 
    \begin{align*}
        & ~ q^{(0)\top} - \eta^{(0)}(q^{(0)\top} X^\top X + \alpha y^\top X) \\
    = & ~ -\alpha (-\frac{1}{\alpha} q^{(0)\top} - \eta^{(0)}(-\frac{1}{\alpha} q^{(0)\top} X^\top X - y^\top X)) \\
    = & ~ -\alpha (\theta^{(0)\top} - \eta^{(0)}(\theta^{(0)\top} X^\top X - y^\top X)) \\
    = & ~ - \alpha \theta^{(1)\top}.
    \end{align*}
    where the first step follows from basic algebra, the second step follows from we defined $\theta^{(0)} = -\frac{1}{\alpha} q^{(0)}$, the third step follows from basic algebra.
    Thus, $q^{(1)} = - \alpha \theta^{(1)}$, so $\theta^{(1)} = -\frac{1}{\alpha} q^{(1)}$. 
    
    Similarly, by math induction, we can have $\theta^{(T)} = -\frac{1}{\alpha} q^{(T)}$. 
    Thus, we finish the proof by  
    $\mathsf{TF}(Z^{(0)}; Q, P) := - q^{(T)}$.
\end{proof}    

\begin{lemma}[Cauchy-Schwarz inequality]\label{lem:cauchy_ineq}
 Let $a,b \in \R^d$, we have
 \begin{align*}
      |\langle a, b \rangle|  \leq \| a \|_2 \cdot \| b \|_2.
 \end{align*}
\end{lemma}

To further refine our understanding of the Looped Transformer's performance, we introduce a bound on the final prediction error:

\begin{lemma}\label{lem:loss}
The final prediction error satisfies 
\begin{align*}
    |\langle \mathsf{TF}(Z^{(0)}; Q, P), \theta^* \rangle - \alpha| \le |\alpha| \cdot \| \theta^{(T)} - \theta^*\|_2.
\end{align*}
\end{lemma}
\begin{proof}
By Definition~\ref{def:icl_task}, we have $q = \mathsf{TF}(Z^{(0)}; Q, P), \theta^* \rangle$. Then, we have
\begin{align*}
    |\langle \mathsf{TF}(Z^{(0)}; Q, P), \theta^* \rangle - \alpha| 
    & ~  = | \langle \alpha \theta^{(T)}, \theta^* \rangle  -\alpha |\\
    & ~ = |\alpha| \cdot | \langle \theta^{(T)}, \theta^* \rangle  - \langle \theta^*, \theta^* \rangle | \\
    & ~ = |\alpha| \cdot | \langle \theta^{(T)} - \theta^*, \theta^* \rangle | \\
    & ~ \le |\alpha| \cdot \| \theta^{(T)} - \theta^*\|_2.
\end{align*}
where the first step is from Definition~\ref{def:icl_task} and Theorem~\ref{thm:transfomer_output}, the second step follows $\|\theta^*\|_2=1$, the third step is from the linear properties of inner product, and the fourth step is from Cauchy-Schwarz inequality (Lemma~\ref{lem:cauchy_ineq}).
\end{proof}

This result provides a quantitative measure of the model's accuracy, linking it directly to the number of iterations and multi-step gradient descent results of linear regression in Eq.~\eqref{eq:gd_update}.

Finally, we present our main theoretical contribution, which encapsulates the core findings of our work:
\begin{theorem}[Main result. Formal version of Theorem~\ref{thm:main:informal}]\label{thm:main}
Let $X \in \R^{n \times d}$, $ y \in \R^{n}$, and $\theta \in \R^d$ be defined in Definition~\ref{def:grad}. 
Let $\kappa$ be the condition number defined in Definition~\ref{def:condition}.
Let $T$ be the number of loops. Let the initial point $q^{(0)} = {\bf 0}_d$. 
Then, we have the final prediction error satisfies 
\begin{align*}
    |\langle \mathsf{TF}(Z^{(0)}; Q, P), \theta^* \rangle - \alpha| \le |\alpha| \cdot  \exp(-\frac{T}{2\kappa}).
\end{align*}
\end{theorem}
\begin{proof}
The proof directly follows Lemma~\ref{lem:loss} and Theorem~\ref{thm:optimz_polyn}.
\end{proof}

In Theorem~\ref{thm:main}, as long as the condition number is constant, we can see that the linear looped Transformer will perform better when the loop number is increasing, i.e., the error will exponentially decay to $0$. 
Usually, we only need $O(d)$ numbers of in-context examples to guarantee a constant $\kappa$. 

The above theorem offers a comprehensive characterization of the Looped Transformer's behavior, providing a tight bound on the prediction error that decays exponentially with the number of iterations. 
The intuition is that the Linear Looped Transformer can explicitly perform gradient descent in its hidden states.
Furthermore, our theoretical finding is also supported by our experiments in Section~\ref{sec:expr}.

\paragraph{Comparison with Previous Works.}
We restate the results in \cite{gsr+24}. 
\begin{theorem}[Theorem 4.1 in \cite{gsr+24}]
   Under condition $\frac{8Td^2}{\sqrt{n}} \leq \frac{1}{2^{2T}}$, we have the optimal linear regression error is $\leq \frac{8Td^2 2^{2T}}{\sqrt{n}}$.
\end{theorem}

In their work, the linear looped transformer has an error bound ${8Td^2 2^{2T}}/{\sqrt{n}}$, while our bound is $|\alpha| \cdot  \exp(-\frac{T}{2\kappa})$.
As the looped number $T$ increases, our error bounds will exponentially decay, while theirs increase exponentially. 
Note that our linear vector generation task and the linear regression task are dual problems.
Our results align with the common intuition that more steps of gradient descent lead to better performance. 
\section{ERROR CONVERGENCE}\label{sec:error}
In this section, we explore the convergence properties of looped transformers, focusing on their behavior under conditions of strong convexity and smoothness. We begin by defining these key concepts and then proceed to establish their implications.
\subsection{Convexity and Smoothness Analysis}\label{sec:error:L_and_mu}
We first introduce some crucial definitions.

\begin{definition}[Strong convexity]\label{def:strong_convex}
    Let $f : \R^d \to \R \cup \{ 
    + \infty \}$ and $\mu > 0$. We say that $f$ is $\mu$-strongly convex if, for every $x,y \in \R^d$, and every $t \in (0 , 1)$ we have that 
    \begin{align*}
        & ~ \mu \frac{\gamma(1- \gamma)}{2} \| x - y\|_2^2 + f(\gamma x + (1- \gamma)y) \\
        \leq & ~ \gamma f(x) + (1- \gamma)f(y).
    \end{align*}
    We say that $\mu$ is the strong convexity constant of $f$.
\end{definition}

\begin{definition}[$L$-smooth]\label{def:l_smooth}
    Let $f : \R^d \to \R$ and $L > 0$. We say that $f$ is $L$-smooth if it is differentiable and if $\nabla f : \R^d \to \R^d$ is $L$-Lipschitz for all $x , y \in \R^d$
    \begin{align*}
        \| \nabla f(x) - \nabla f(y) \|_2 \leq L \| x -y \|_2.
    \end{align*}
\end{definition}

To quantitatively analyze the parameter dynamics in our linear vector generation task, we first derive the Lipschitz and convexity constants for the model introduced in Definition~\ref{def:grad}.
\begin{lemma}\label{lem:grad_l_ell}
     Given $X \in \R^{n \times d}$, $ y \in \R^{n}$, and $\theta \in \R^d$ in Definition~\ref{def:grad}, we have 
    \begin{align*}
        L = \| X^\top X \|
    \end{align*}
    where $L$ is the Lipschitz constant defined in Definition~\ref{def:l_smooth}.
\end{lemma}
\begin{proof}
    From Definition~\ref{def:grad}, we have 
    \begin{align*}
        \ell (\theta) = & ~0.5 \| y - X\theta\|_2^2.
    \end{align*}
    The gradient will be
    \begin{align*}
        \nabla_\theta \ell (\theta) = & ~ X^\top X \theta - X^\top y \in \R^d.
    \end{align*}
    Then for $\theta_1, \theta_2 \in \R^d$, we have
    \begin{align*}
        \| \nabla \ell_{\theta_1}(\theta_1) - \nabla \ell_{\theta_2}(\theta_2) \|_2 = & ~ \| X^\top X(\theta_1 -\theta_2) \|_2\\
        \leq & ~ \| X^\top X\| \cdot \| \theta_1 - \theta_2\|_2 
    \end{align*}
    where the first step follows from Definition~\ref{def:grad}, and the second step follows from properties of the norm. From Definition~\ref{def:l_smooth}, we observe that
    \begin{align*}
        L = \| X^\top X\|,
    \end{align*}
    where $\| X^\top X\|$ is the spectral norm of $ X^\top X$ denoting the maximum eigenvalue. 
\end{proof}

The following two lemmas are closely related and build upon each other to establish the strong convexity constant for a specific optimization problem.

\begin{lemma}[Forklore]\label{lem:norm_eigenvalue_convert} For $X \in \R^{n \times d}$ and $v\in \R^d$, we have
\begin{align*}
    v^\top (X^\top X) v \geq \lambda_{\min}(X^\top X) \| v \|_2^2.
\end{align*}
\end{lemma}
\begin{proof}
Let $\{\lambda_i\}_{i=1}^d$ be the eigenvalue of $X^\top X$, $\{v_i\}_{i=1}^d$ is the corresponding standard orthogonal feature vector, we have
\begin{align*}
v^\top (X^\top X) v = & ~ (\sum_{i=1}^d c_i v_i^\top) (X^\top X) (\sum_{j=1}^d c_j v_j) \\
= & ~ \sum_{i=1}^d \sum_{j=1}^d c_i c_j v_i^\top (X^\top X) v_j \\
= & ~ \sum_{i=1}^d \sum_{j=1}^d c_i c_j \lambda_j v_i^\top v_j \\
= & ~ \sum_{i=1}^d c_i^2 \lambda_i \\
\geq & ~ \lambda_{\min}(X^\top X) \sum_{i=1}^d c_i^2 \\
= & ~ \lambda_{\min}(X^\top X) \|v\|_2^2
\end{align*}
where the first step follows from properties of symmetric matrix, the second step follows from the matrix being decomposed into linear combinations of eigenvectors, the third step follows from basic algebra, the fourth step follows from orthogonality of eigenvectors, the fifth step follows from $\lambda_{\min}$ denotes the minimum eigenvalue, the sixth step follows from the definition of norm.
\end{proof}

\begin{lemma}\label{lem:mu_of_ell}
    Given $X \in \R^{n \times d}$, $ y \in \R^{n}$, and $\theta \in \R^d$ in Definition~\ref{def:grad}, we have 
    \begin{align*}
        \mu = \lambda_{\min}(X^\top X)
    \end{align*}
    where $\mu$ is the strong convexity constant defined in Definition~\ref{def:strong_convex}
\end{lemma}
\begin{proof}
    For any $\theta_1, \theta_2 \in \R^d$ and $t \in (0,1)$, we have
    \begin{align*}
        & ~ \ell (\gamma \theta_1 + (1-\gamma) \theta_2) \\
        = & ~ 0.5 \| y - X(\gamma \theta_1 + (1-\gamma) \theta_2) \|_2^2\\
        = & ~0.5 \| \gamma(y - X\theta_1) + (1-\gamma)(y - X\theta_2) \|_2^2\\
        \leq & ~ 0.5(\gamma \| y - X\theta_1 \|_2^2 + (1-\gamma)\| y - X\theta_2 \|_2^2) \\
        \quad & - 0.5 \gamma(1 - \gamma) \| X(\theta_1 - \theta_2) \|_2^2\\
        = &~ 0.5(\gamma \| y - X\theta_1 \|_2^2 + (1-\gamma)\| y - X\theta_2 \|_2^2)\\
        \quad & - 0.5 \gamma(1-\gamma)(\theta_1 - \theta_2)^T X^T X(\theta_1 - \theta_2)\\
        \leq & ~ 0.5(\gamma \| y - X\theta_1 \|_2^2 + (1-\gamma)\| y - X\theta_2 \|_2^2)\\ 
        \quad & - 0.5 \gamma(1-\gamma)\lambda_{\min}(X^T X)\|\theta_1 - \theta_2\|_2^2
    \end{align*}
     where the first step follows from Definition~\ref{def:grad}, the rest step follow from basic algebra and Lemma~\ref{lem:norm_eigenvalue_convert}. From Definition~\ref{def:strong_convex}, we observe that
    \begin{align*}
        \mu = \lambda_{\min}(X^\top X)
    \end{align*}
   where $ \lambda_{\min}(X^\top X)$ denotes the minimum eigenvalue of $X^\top X$.
\end{proof}

\subsection{Main Result}\label{sec:error:upper_bound}
We first commence with a statement of Lemma~\ref{lem:converge_strong_l_smooth}, which furnishes a convergence rate for gradient descent on strongly convex and smooth functions. 
\begin{lemma}[Theorem~3.6 in~\cite{gg23}]\label{lem:converge_strong_l_smooth}
    Let $\ell:\R^d \to \R$ be a differentiable function and assume $\ell$ is $\mu$-strongly convex and $L$-smooth. Let $\{\theta^{(t)}\}_{t \in \mathbb N}$ be the sequence generated by the gradient descent algorithm, with a stepsize $\eta \in (0, \frac{1}{L}]$. Then for $\theta^* = \arg\min_{\theta} \ell(\theta)$ and for all $t \in \mathbb N$, we have
    \begin{align*}
        \|\theta^{(t)} - \theta^* \|_2^2 \leq (1-\eta \mu)^{t} \|\theta^{(0)} - \theta^*\|_2^2.
    \end{align*}
\end{lemma}

\begin{fact}[Folklore]\label{fac:ex_to_poly_ineq}
    For any $n, T \in \N_+$, we have
    \begin{align*}
        (1 - \frac{1}{n})^T \leq e^{-T/n}.
    \end{align*}
\end{fact}

We now present a rigorous upper bound on the error magnitude of the gradient descent algorithm's output after $T$ iterations, elucidating the convergence properties of this optimization method in the context of linear vector generation.
\begin{theorem}
\label{thm:optimz_polyn}
If the following holds:
\begin{itemize}
    \item Let $T$ be the loop number.
    \item Let $n, T, t \in \mathbb N$.
    \item Let $X \in \R^{n \times d}$, $ y \in \R^{n}$, and $\theta \in \R^d$ be defined in Definition~\ref{def:grad}.
    \item Let the condition number $\kappa = \frac{\lambda_{\max}(X^\top X)}{\lambda_{\min}(X^\top X)}$.
    \item The step size $\eta = \frac{1}{L}$.
    \item Let $\mu$ and $L$ be defined in Definition~\ref{def:strong_convex} and Definition~\ref{def:l_smooth}.
    \item  For $\ell(\theta)$ defined in Definition~\ref{def:grad}, we have $\ell(\theta)$ is $L$-smooth and $\mu$ strong convex, where $L = \| X^\top X\|$ and $\mu = \lambda_{\min}(X^\top X)$.
    \item The initial point $\theta^{(0)}$ satisfies $\|\theta^{(0)} - \theta^*\|_2 \leq R$.
\end{itemize}
Then, we have
\begin{align*}
    \|\theta^* - \theta^{(T)} \|_2^2 \leq e^{-T/\kappa} R^2 . 
\end{align*}
\end{theorem}
\begin{proof}
First, we have
\begin{align*}
    \|\theta^{(t)} - \theta^* \|_2^2 \leq & ~ (1-\eta \mu)^{t}     \|\theta^{(0)} - \theta^*\|_2^2\\
    \leq & ~(1-\frac{\mu}{L})^{t} \|\theta^{(0)} -    \theta^*\|_2^2
\end{align*}
where the first step follows from Lemma~\ref{lem:converge_strong_l_smooth}, the second step follows from we choose $\eta^{(t)} = \frac{1}{L}$. Then consider the term $\frac{\mu}{L}$, we have
\begin{align*}
   \frac{\mu}{L} = & ~\frac{\lambda_{\min}(X^\top X)}{\| X^\top X\|} = \frac{1}{\kappa}
\end{align*}
where the first step follows from Lemma~\ref{lem:grad_l_ell} and Lemma~\ref{lem:mu_of_ell}, the second step follows the definition of condition number $\kappa = \frac{\lambda_{\max}(X^\top X)}{\lambda_{\min}(X^\top X)}$.

Then, substituting this back, we have
\begin{align*}
    \|\theta^{(T)} - \theta^* \|_2^2 \leq & ~(1-\frac{1}{\kappa})^{T} \|\theta^{(0)} - \theta^*\|_2^2\\
    = & ~ (1-1/\kappa)^{\kappa \cdot T/\kappa} \|\theta^{(0)} - \theta^*\|_2^2 \\
    \leq & ~ e^{-T/\kappa} \|\theta^{(0)} - \theta^*\|_2^2  \\
    \leq & ~ e^{-T/\kappa} R^2
\end{align*}
where the first step follows from basic algebra, the second step follows from $(1-1/\kappa)^\kappa < e^{-1}$ (Fact~\ref{fac:ex_to_poly_ineq}), and the last step follows from $\|\theta^{(0)} - \theta^*\|_2 \leq R$.
\end{proof}

Theorem~\ref{thm:optimz_polyn} tells us that GD can well solve linear regression tasks. In particular, when the input data has a good condition number, the approximation error will exponentially decay to $0$. We use the above insights in the proof of our main results (Theorem~\ref{thm:main}).

\section{EXPERIMENTS}\label{sec:expr}
In this section, we aim to verify our theory by evaluating the convergence behavior of gradient descent for linear vector generation. We designed our experiment to examine the impact of varying sample sizes on convergence rates while keeping the feature dimension fixed. Our results demonstrate that empirical convergence rates consistently outperform theoretical upper bounds across all sample sizes, with significant improvement in convergence speed as the condition number decreases, validating our theoretical predictions.

\begin{figure}[!ht]
\centering
\includegraphics[width=0.5\textwidth]{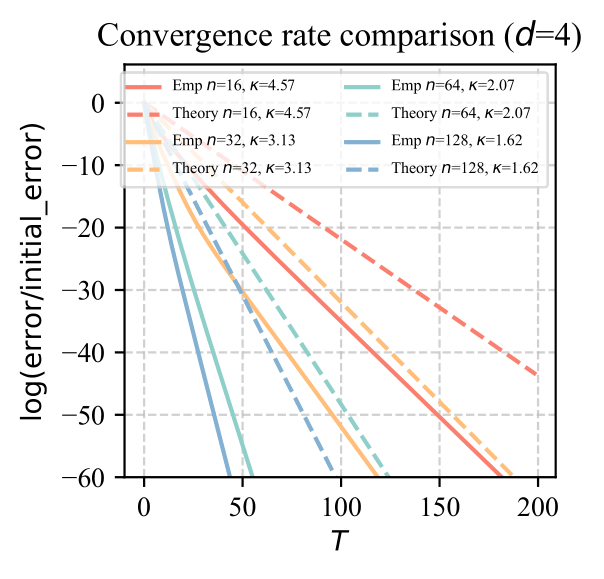}
\caption{
The convergence rate comparison for gradient descent in linear vector generation with a fixed dimension $d=4$ and varying sample sizes $n \in \{ 16, 32, 64, 128 \}$ and their corresponding condition number $\kappa$. 
The `Emp' means the empirical error of our experiments. The `Theory' means the theoretical bound in Theorem~\ref{thm:main}.
The $y$-axis is the logarithm of normalized error and the $x$-axis is the number of loops $T$.
Both empirical (solid lines) and theoretical (dashed lines) results are presented for each $n$. The plot demonstrates that as the sample size $n$ increases, the condition number will decrease, so the convergence rate improves. Thus, with larger $n$ values, there will be a steeper slope and faster convergence to the optimal solution.
}
\label{fig:convergence}
\end{figure}

\subsection{Experiment Setup} 
In this experiment, we aimed to investigate the convergence behavior of multi-step gradient descent for Linear Looped Transformer (Definition~\ref{def:loop}) in-context learning the linear vector generation task (Definition~\ref{def:icl_task}). 
We randomly draw each entry of $X \in \mathbb{R}^{n \times d}$ from standard Gaussian distribution, $\mathcal{N}(0,1)$, and response variables $y = X\theta^*$, where $\theta^* \in \mathbb{R}^d$ was randomly chosen. Our experiments focused on scenarios with $d$ fixed at $4$ and $n$ varying in $\{16, 32, 64, 128\}$. For each $(n,d)$ combination, we implemented gradient descent with $T = 200$ iterations and learning rate $\eta = 1/L$, where $L = \|X^\top X\|$. To ensure statistical robustness, we conducted 10 independent trials for each configuration. Convergence was measured by tracking $\|\theta^{(t)} - \theta^*\|_2^2$, where $\theta^*$ is the optimal least squares solution. To facilitate comparison across different problem sizes, we normalized error plots by the initial error, plotting $\log(\|\theta^{(t)} - \theta^*\|_2^2 / \|\theta^{(0)} - \theta^*\|_2^2)$ against $t$. This normalization enabled a clear comparison of relative decay rates, irrespective of initial error magnitudes, thus providing insights into the impact of condition number on gradient descent convergence in the linear vector generation task.

\subsection{Result Interpretation}
Our experiment investigates the convergence behavior of gradient descent for linear vector generation with varying sample sizes $n$ and a fixed feature dimension $d=4$. Figure~\ref{fig:convergence} illustrates the convergence rates for different $n$ values $\{ 16, 32, 64, 128 \}$, comparing empirical results with theoretical bounds in Theorem~\ref{thm:main}.
A key aspect of this experiment is the condition number $\kappa$, which decreases as the number of examples increases. The average $\kappa$ values for the different sample sizes are $\{ 4.57, 3.13, 2.07, 1.62 \}$, corresponding to $n = \{ 16, 32, 64, 128 \}$ respectively. This inverse relationship between $n$ and $\kappa$ is noteworthy, as it significantly influences the convergence rates.
The results demonstrate that as the sample size $n$ increases, the convergence rate improves substantially. This is evident from the steeper slopes of both empirical and theoretical lines for larger $n$ values. Importantly, the empirical convergence rates consistently outperform the theoretical upper bounds across all sample sizes, with the gap between empirical and theoretical performance narrowing as $n$ increases.
This observation aligns with our theoretical expectations in Theorem~\ref{thm:main} and highlights the crucial role of the condition number in determining convergence behavior.

\section{CONCLUSION}\label{sec:conclusion}
In this work, we have demonstrated that linear looped Transformers can efficiently implement multi-step gradient descent for in-context learning, requiring only a reasonable number of examples when input data is well-conditioned. This finding relieves the previous assumptions of an exponential number of in-context examples and offers new insights into the capabilities of Transformer architectures. Our theoretical analysis and preliminary experiments pave the way for more efficient inference algorithms in large language models and open avenues for future research in this domain.

\ifdefined\isarxiv
\section*{Acknowledgments}
Research is partially supported by the National Science Foundation (NSF) Grants 2023239-DMS, CCF-2046710, and Air Force Grant FA9550-18-1-0166.
\bibliographystyle{alpha}
\bibliography{ref}
\else

\section*{Acknowledgments}
Research is partially supported by the National Science Foundation (NSF) Grants 2023239-DMS, CCF-2046710, and Air Force Grant FA9550-18-1-0166.
\bibliography{ref}
\bibliographystyle{plainnat}
\input{checklist}
\fi

\newpage
\onecolumn
\appendix





\end{document}